\documentclass{IEEEtran}
\usepackage{cite}
\usepackage{amssymb}
\usepackage{algorithmic}
\usepackage{textcomp}
\usepackage{graphicx}
\usepackage{xcolor}
\definecolor{BrightOrange}{HTML}{FF7A00}  

\usepackage{tikz}
\usetikzlibrary{shapes,arrows}
\usetikzlibrary{positioning,calc}

\usepackage{bbm}
\usepackage{mathrsfs}
\usepackage{verbatim}
\usepackage{amsmath}
\usepackage{amsfonts}
\usepackage{bm}
\usepackage{caption}
\usepackage{float}

\usepackage{amsthm}
\usepackage{mathtools}
\usepackage{cite}
\usepackage{hyperref}

\usepackage{subfigure}
\allowdisplaybreaks

\theoremstyle{definition}
\newtheorem{assumption}{Assumption}

\newtheorem{remark}{Remark}
\newtheorem{problem}{Problem}

\theoremstyle{plain}

\newtheorem{theorem}{Theorem}
\newtheorem{lemma}{Lemma}

\def\BibTeX{{\rm B\kern-.05em{\sc i\kern-.025em b}\kern-.08em
    T\kern-.1667em\lower.7ex\hbox{E}\kern-.125emX}}
\title{Model-Based Reinforcement Learning Under Confounding}
\author{Nishanth Venkatesh$^{1}$, \IEEEmembership{Student Member, IEEE}, and Andreas A. Malikopoulos$^{1,2}$, \IEEEmembership{Senior Member, IEEE}
\thanks{This research was supported by in part by NSF under Grants CNS-2401007, CMMI-2348381, IIS-2415478 and in part by MathWorks.}
    \thanks{$^{1}$Department of Systems Engineering, Cornell University, Ithaca, NY 14850 USA.}
    \thanks{$^{2}$School of Civil and Environmental Engineering, Cornell University, Ithaca, NY 14853 USA (email: \texttt{ns942@cornell.edu; amaliko@cornell.edu).}} }
\begin{document}

\maketitle

\begin{abstract}
We investigate model-based reinforcement learning in contextual Markov decision processes (C-MDPs) in which the context is unobserved and induces confounding in the offline dataset. In such settings, conventional model-learning methods are fundamentally inconsistent, as the transition and reward mechanisms generated under a behavioral policy do not correspond to the interventional quantities required for evaluating a state-based policy. To address this issue, we adapt a proximal off-policy evaluation approach that identifies the confounded reward expectation using only observable state-action-reward trajectories under mild invertibility conditions on proxy variables. When combined with a behavior-averaged transition model, this construction yields a surrogate MDP whose Bellman operator is well defined and consistent for state-based policies, and which integrates seamlessly with the maximum causal entropy (MaxCausalEnt) model-learning framework. The proposed formulation enables principled model learning and planning in confounded environments where contextual information is unobserved, unavailable, or impractical to collect.
\end{abstract}


\section{Introduction}

Model-based reinforcement learning (MBRL) has long been recognized as a data-efficient approach to sequential decision-making, as it explicitly estimates the system dynamics and reward mechanism to enable planning through simulated rollouts \cite{Sutton1991Dyna}. This paradigm typically achieves greater sample efficiency than model-free methods. Recent developments include the use of Gaussian process models for dynamics learning \cite{Deisenroth2011PILCO} and neural dynamics models with model-free fine-tuning \cite{Nagabandi2018Neural}, both of which demonstrate improved data efficiency in policy optimization. Additional advances based on probabilistic ensembles with trajectory sampling \cite{Chua2018PETS} and latent-dynamics planning \cite{Hafner2019PlaNet} further illustrate that learned models can support competitive performance on complex control tasks. A common requirement across these methodologies is the ability to learn a model that accurately captures the underlying transition and reward structures, as model fidelity is essential for reliable planning.

Most MBRL methods implicitly assume that the data-generating process is fully observed and free of hidden factors that jointly influence decisions and outcomes.
In many real-world domains, this assumption is violated.
Examples include human--robot interaction, routing and maneuvering connected and automated vehicles in mixed traffic \cite{venkatesh2023connected}, medical records, and personalized decision systems \cite{dave2024airecommend,sun2025recommendation}.
In such settings, the offline dataset available for learning may omit latent variables that affect both the environment and the data-generating policy. When actions depend on these hidden variables, the collected data encode arbitrary correlations among state, action, next state, and reward.
In causal inference, this phenomenon is referred to as \emph{confounding}.
Consequently, transition and reward models learned directly from such data could misrepresent the effect of candidate policies and introduce systematic bias into planning.

For example, consider historical medical records in which each entry contains a doctor’s chosen treatment and the patient’s eventual outcome. Although the dataset includes observable patient characteristics, the doctor’s decisions were based on additional information—such as severity of symptoms, early warning signs, or clinical intuition—that was not recorded. This latent factor affects both the treatment selected and the patient’s likelihood of recovery, producing correlations that reflect the doctor’s unobserved reasoning rather than the causal effect of the treatment itself. Learning treatment effects directly from such observational data leads to biased conclusions, since the behavioral policy exploited information unavailable to the learner. This phenomenon mirrors the confounding encountered in contextual decision processes, where unobserved context simultaneously influences the behavioral policy and the system response.

Reinforcement learning under unobserved confounding has been studied primarily from a model-free perspective.
Identification and proximal causal inference frameworks \cite{tchetgen2020introduction,cui2020semiparametric,miao2018identifying}
provide tools to recover causal effects using proxy variables when hidden confounders are present.
These include robust off-policy evaluation (OPE) under bounded confounding \cite{kallus2020confounding}, non-identifiability and hardness results for confounded environments \cite{Makar2022ConfoundingRL}, and proximal
reinforcement learning methods that exploit proxy variables \cite{bennett2024proximal}.
Recent developments provide tensor-based identification policies for offline data \cite{kausik2024offline,Kausik2024OfflineConfounding} and proximal off-policy evaluation in partially observed Markov decision processes \cite{tennenholtz2020off}.
These contributions provide powerful tools for off-policy evaluation and policy optimization.
However, they largely operate in a model-free setting and do not directly address learning a surrogate transition-reward model that can be embedded into standard MBRL pipelines.

Contextual Markov decision processes (C-MDPs) provide a natural setting for confounded MBRL, as an unobserved context variable can influence both the transition law and the reward function.
When the behavioral policy that generated the dataset depends on this hidden context, the observed transitions and rewards reflect correlations induced by the unobserved variable rather than the causal effect of the actions. As a result, models learned directly from such behavioral data fail to represent the system from the perspective of a state-based evaluation policy, and the Bellman equations associated with that policy are no longer consistent with the available observations. A central challenge, therefore, is to construct a surrogate MDP whose transition dynamics and reward expectations are identifiable from the restricted dataset and faithfully capture the environment as experienced by a state-based agent.

In this paper, we develop an MBRL framework for contextual MDPs in which the context is unobserved in the available dataset and acts as a confounder.
Our approach combines a behavior-averaged transition model with a proximal reconstruction of the reward component of the Bellman equation.
We adapt a proximal OPE construction technique \cite{tennenholtz2020off} to the C-MDP setting to express the confounded reward expectation as a functional of observable proxy variables under standard invertibility assumptions from proximal causal inference.
This yields a deconfounded Bellman operator that is compatible with state-based policies and can be evaluated from offline data.
We adapt a well-known data-based model-learning framework- maximum causal entropy (MaxCausalEnt), introduced in \cite{ziebart2010modeling}.
Inverse reinforcement learning and belief-dynamics extensions to the MaxCausalEnt framework in \cite{herman2016inverse,Reddy2018BeliefDynamics} demonstrate how such formulations can capture purposeful behavior in stochastic environments.
We integrate the proximal reward correction with the MaxCausalEnt model-learning formulation.
In our setting, the MaxCausalEnt framework is used to jointly learn a Q-function and a behavior-averaged transition model that are Bellman-consistent with respect to the deconfounded reward expectation.
The resulting surrogate MDP supports model-based planning in confounded environments where contextual information is unavailable or impractical to record.

The remainder of the paper is organized as follows. 
In Section~\ref{section:Problem Formulation}, we formalize the learning problem that arises when data are generated by an unknown, context-dependent behavioral policy, and we introduce the resulting surrogate model-learning formulation. In Section~\ref{section:Solution Approach}, we develop our solution approach: we reinterpret the C--MDP through a causal inference perspective, reformulate the system as a POMDP to reveal the role of the hidden context, and derive a proximal off-policy evaluation method that identifies the confounded reward expectation using observable proxy variables. In Section~\ref{section:Numerical Simulation}, we present a numerical example that illustrates the necessity and impact of the proximal correction. Finally, in Section~\ref{sec:conclusion}, we draw concluding remarks and outline several directions for future research.

\section{Problem Formulation}
\label{section:Problem Formulation}

We consider a system that can be modeled as a C-MDP given by the tuple $\mathcal{M}= ( \mathcal{X}, \mathcal{Z}, \mathcal{U}, P, r,\eta,\nu,T )$. 
The system evolves over discrete time steps $t = 0,\ldots,T$, where $T \in \mathbb{N}$ is the finite horizon. The state space $\mathcal{X}$ and action space $\mathcal{U}$ are finite. At each time $t$, the random variables $X_t \in \mathcal{X}$ and $U_t \in \mathcal{U}$ denote the system state and control action, respectively, and the initial state satisfies $X_0 \sim \eta$. The system is influenced by an unobserved context $Z_t \in \mathcal{Z}$, with $Z_0 \sim \nu$. The joint process $(X_t, Z_t)$ is Markovian with respect to $(X_t, Z_t, U_t)$. The state transition is governed by
\[
P_X : \mathcal{X} \times \mathcal{U} \times \mathcal{Z} \rightarrow \Delta(\mathcal{X}), 
\qquad X_{t+1} \sim P_X(\cdot \mid x_t, z_t, u_t),
\]
and the context evolves according to
\[
P_Z : \mathcal{Z} \times \mathcal{X} \times \mathcal{U} \rightarrow \Delta(\mathcal{Z}), 
\qquad Z_{t+1} \sim P_Z(\cdot \mid x_t, z_t, u_t).
\]
The context $Z_t$ is unobserved by the planner, and it may influence both dynamics and rewards. 
Furthermore, the reward function is given by the mapping $r: \mathcal{X} \times \mathcal{U} \times \mathcal{Z} \rightarrow \mathbb{R}$.
At each $t=0,\ldots, T-1$, we let $\tau_t$ be the realization of a trajectory of the system up to time $t$, which we define as 
\begin{align}
    \tau_t&=(x_0,z_0,u_0,r_0,\ldots,x_t,z_t,u_t,r_t), 
\end{align}
where $\tau= (x_0,z_0,u_0,r_0,\ldots,x_T,z_T,u_T,r_T)$ denotes the entire trajectory up to the time of horizon.
Let $\mathcal{T}$ be the space of all feasible full trajectories $\tau$ of the system. 
Having specified the C-MDP model, we now formalize the decision-making problem of interest within this framework.

\subsection{State-based Control policy}
\label{section:Problem Formulation}

A planner has access to a dataset generated by an expert who interacts with the C-MDP
described above. 
We consider that the expert follows a behavioral policy 
$\boldsymbol{g}^{\boldsymbol{b}}=\{g_t^{\boldsymbol{b}}\}_{t=0}^{T}$, where $g_t^{\boldsymbol{b}}$ is the expert's control law at time $t$. 
Each control law $g_t^{\boldsymbol{b}}$ may depend on the unobserved context $Z_t$ and can therefore be of the form 
$g_t^{\boldsymbol{b}}:\mathcal{X}\times\mathcal{Z}\rightarrow\Delta(\mathcal{U})$.
Since contextual information is not recorded, the planner observes only $(x_t,u_t,r_t)$ tuples.  
Let $\tau^o=(x_0,u_0,r_0,\ldots,x_T,u_T,r_T)$ denote the observable component  
of a trajectory $\tau$, and let 
$\mathcal{D}^{\boldsymbol{b}}=\{\tau^{o,i}\}_{i=1}^{N}$ be the dataset.

Since the expert’s behavior may depend on the latent context, two key components of the C-MDP become unidentifiable from the observable data:  
(i) the context-dependent transition mapping $p_x(x_{t+1}\mid x_t,z_t,u_t)$, and  
(ii) the expert's full state-context-based control law $g_t^{\boldsymbol{b}}(u_t\mid x_t,z_t)$ at each $t$.  
However, although the context is unobserved, the behavioral policy induces a well-defined  \emph{state-action next-state} distribution, which is identifiable from the dataset.
Thus, based on the triplets $(x_t,u_t,x_{t+1})$ from the dataset, the context-marginalized transition model 
\begin{align}
    p^{\boldsymbol{\phi}}(x_{t+1}\mid x_t,u_t)
    =\sum_{z_t\in\mathcal{Z}}
    p(x_{t+1}\mid x_t,z_t,u_t)\,p(z_t\mid x_t,u_t). 
    \label{eq:marg_transition}
\end{align}
is identifiable and parameterized by a stochastic matrix $\boldsymbol{\phi}$ of appropriate dimension, which serves as a surrogate that approximates the context-marginalized dynamics.
Each element of the matrix $\boldsymbol{\phi}$ is denoted by $p^{\boldsymbol{\phi}}(x_{t+1}\mid x_t,u_t)$.
As the context is unobserved, the planner cannot recover the true context-dependent model $p_X(x_{t+1}\mid x_t,z_t,u_t)$.
Instead, the learned model represents the \emph{context-marginalized} evolution of the state, induced by the expert.

In parallel, the behavioral policy induces a distribution over actions conditioned on the observable state.  
The planner, therefore, seeks to learn a state-based control policy that best explains the observed actions under a suitable modeling framework.
In particular, the planner aims to learn  
(i) a transition model $p^{\boldsymbol{\phi}}(x_{t+1}\mid x_t,u_t)$ from observable transitions, and  
(ii) a state-based policy 
$\boldsymbol{g}^{\boldsymbol{\theta}}=\{g_t^{\boldsymbol{\theta}}\}_{t=0}^{T-1}$ indexed by $\theta$, where $\theta$ identifies a particular member of the policy class under consideration. For each $t$, the mapping $g^{\theta}_t : \mathcal{X} \rightarrow \Delta(\mathcal{U})$ specifies a distribution over actions conditioned only on the observable state.

\subsection{State-Based Policy and Value/Q Parameterization}

We restrict attention to state-based policies 
$\boldsymbol{g}^{\boldsymbol{\theta}}=\{g_t^{\boldsymbol{\theta}}\}_{t=0}^{T}$,
where the control law at each $t$ is a mapping
$g_t^{\boldsymbol{\theta}}:\mathcal{X}\rightarrow\Delta(\mathcal{U})$.
For any state $x_t$ and action $u_t$, we denote the action probability by
$p^{\boldsymbol{\theta}}(u_t\mid x_t)$.
Given a surrogate transition model $p^{\boldsymbol{\phi}}$, the value of
policy $\boldsymbol{g}^{\boldsymbol{\theta}}$ at time $t$ is
\begin{align}
    V_t^{\boldsymbol{\theta},\boldsymbol{\phi}}(x_t)
    =
    \mathbb{E}^{\boldsymbol{\theta},\boldsymbol{\phi}}
    \!\left[
        \sum_{k=t}^{T-1} r(X_k,Z_k,U_k)
        \,\Big|\, X_t=x_t
    \right].
    \label{eq:def_value_function}
\end{align}

\noindent The corresponding Q-function is defined as
\begin{align}
    Q_t^{\boldsymbol{\theta},\boldsymbol{\phi}}(x_t,u_t)
    &= 
    \mathbb{E}^{\boldsymbol{\theta}}
    \!\left[
        r(X_t,Z_t,U_t)
        \,\big|\, X_t=x_t,U_t=u_t
    \right] \nonumber\\[2pt]
    &\quad +
    \mathbb{E}^{\boldsymbol{\phi}}
    \!\left[
        V_{t+1}^{\boldsymbol{\theta},\boldsymbol{\phi}}(X_{t+1})
        \,\big|\, X_t=x_t,U_t=u_t
    \right],
    \label{eq:def_Q_function}
\end{align}
where $\mathbb{E}^{\boldsymbol{\theta}}[\cdot]$ denotes expectation with
respect to the distribution of the latent context $Z_t$ induced by the
policy $\boldsymbol{g}^{\boldsymbol{\theta}}$, and
$\mathbb{E}^{\boldsymbol{\phi}}[\cdot]$ denotes expectation with respect
to the stochastic evolution of $X_{t+1}$ under the surrogate transition
model ${\boldsymbol{\phi}}$.
The value and Q-functions satisfy
\begin{align}
    V_t^{\boldsymbol{\theta},\boldsymbol{\phi}}(x_t)
    =
    \sum_{u\in\mathcal{U}}
    p^{\boldsymbol{\theta}}(u\mid x_t)\,
    Q_t^{\boldsymbol{\theta},\boldsymbol{\phi}}(x_t,u).
    \label{eq:V_Q_relation}
\end{align}

To jointly learn the surrogate transition model and the state-based
policy, we adopt the maximum causal entropy framework
\cite{ziebart2010modeling}.  
This framework provides a data-driven methodology to infer both the dynamics and
the behavioral policy from a dataset of trajectories.  
Among all models that explain the observed actions, this framework selects one that  
(i) matches the expert’s action distribution as closely as possible, while  
(ii) avoiding additional structure not justified by the data.
In a finite-horizon setting, this principle leads to a softmax (Boltzmann)
representation of the policy in terms of the Q-function under the
surrogate model.

Hence, for a fixed surrogate transition model ${\boldsymbol{\phi}}$, the
state-based control law at time $t$ is parameterized as
\begin{align}
    p^{\boldsymbol{\theta}}(u_t\mid x_t)
    =
    \frac{
        \exp\!\big(Q_t^{\boldsymbol{\theta},\boldsymbol{\phi}}(x_t,u_t)\big)
    }{
        \sum_{u'\in\mathcal{U}}
        \exp\!\big(Q_t^{\boldsymbol{\theta},\boldsymbol{\phi}}(x_t,u')\big)
    }.
    \label{eq:softmax_policy}
\end{align}

\noindent For each fixed $\boldsymbol{\phi}$, the Q-function is determined by the parameters $(\boldsymbol{\theta},\boldsymbol{\phi})$,
and the softmax relation \eqref{eq:softmax_policy} induces a unique
state-based control law $g_t^{\boldsymbol{\theta}}$.
The policy is indexed by $\boldsymbol{\theta}$ to emphasize that
$\boldsymbol{\phi}$ is treated as part of the learned surrogate
environment, while $\boldsymbol{\theta}$ parameterizes the choice of
policy.
Thus, learning $(\boldsymbol{\theta},\boldsymbol{\phi})$ amounts to
learning a Q-function that is Bellman-consistent with respect to
$p^{\boldsymbol{\phi}}$ and then using \eqref{eq:softmax_policy} to
obtain the corresponding state-based policy.

\subsection{MaxCausalEnt Model Learning}

The MaxCausalEnt principle fits $(\boldsymbol{\theta},\boldsymbol{\phi})$
by maximizing the likelihood of expert actions under the softmax policy while 
enforcing Q-based Bellman consistency.

\begin{problem}[MaxCausalEnt Model Learning]
\label{prob:maxcausalent}
\begin{align}
    \max_{\boldsymbol{\theta},\boldsymbol{\phi}}
    \quad &
    \mathbb{E}_{\mathcal{D}^{\boldsymbol{b}}}
    \!\left[
        -\log \!\left(
            \prod_{t=0}^{T-1}
            p^{\boldsymbol{\theta}}(u_t\mid x_t)
        \right)
    \right]
    \label{eq:maxent_obj}
\end{align}
\begin{align}
    \text{subject to}\quad
    Q_t^{\boldsymbol{\theta},\boldsymbol{\phi}}(x_t,u_t)
    &=
    \mathbb{E}^{\boldsymbol{\theta}}
    \!\left[
        r(X_t,Z_t,U_t)
        \mid x_t,u_t
    \right] \nonumber\\[-2pt]
    &\quad+
    \mathbb{E}^{\boldsymbol{\phi}}
    \!\left[
        V_{t+1}^{\boldsymbol{\theta},\boldsymbol{\phi}}(X_{t+1})
        \mid x_t,u_t
    \right],
    \label{eq:maxent_bellman_Q}\\[4pt]
    V_t^{\boldsymbol{\theta}}(x_t)
    &= \sum_{u_t\in\mathcal{U}}
       p^{\boldsymbol{\theta}}(u_t\mid x_t)\,
       Q_t^{\boldsymbol{\theta},\boldsymbol{\phi}}(x_t,u_t), \nonumber\\
        \forall x_t&\in\mathcal{X},\;
        \forall u_t\in\mathcal{U}, \nonumber
\end{align}
where \eqref{eq:maxent_bellman_Q} expresses the Bellman consistency of the parametrized policy $\boldsymbol{g}^{\boldsymbol{\theta}}$ under the surrogate model $\boldsymbol{\phi}$.
\end{problem}

\begin{remark}
If contextual information were available, the reward and transition 
expectations in~\eqref{eq:maxent_bellman_Q} could be computed by 
marginalizing over $(x_t,z_t,u_t,x_{t+1})$ tuples.  
\end{remark}


\section{Solution Approach}
\label{section:Solution Approach}

In this section, we first reinterpret the C-MDP from a causal inference perspective. This highlights how the latent context acts as an unobserved confounder.
We cast the system as a structural causal model (SCM) to clarify how the behavioral policy and any parameterized policy induce different causal graphs and, consequently, different trajectory distributions. 
This motivates the use of \emph{proximal learning}—a causal inference tool designed to recover causal effects in the presence of hidden confounding—as a means to compute the reward term in the Bellman equation using only observable variables.
We then introduce the notation needed for the analysis and reinterpret the C-MDP as a POMDP to make the role of the hidden context explicit.
Later, we develop an off-policy evaluation-based modification of the MaxCausalEnt framework that uses proximal identification to obtain a deconfounded Bellman expectation specific to our problem formulation.

\subsection{Causal Inference Perspective}

Causal inference provides tools for reasoning about how interventions change the behavior of a system.
This allows analysis of the system beyond what is revealed by observational or correlational structure alone.
This distinction is especially important in offline settings—such as the C-MDP that we consider.
It is essential to note that the offline data collected under a behavioral policy
may rely on additional information unavailable to the planner, as given in Problem \ref{prob:maxcausalent}. 
In such cases, functional relationships between variables fail to reveal how the system would evolve under a different (intervened) policy.

A SCM represents each variable in the system as the output of a structural equation, with a directed acyclic graph (DAG) specifying the underlying causal dependencies \cite{pearl2009causality}. For the C--MDP introduced in Section~\ref{section:Problem Formulation}, these structural equations coincide with the stochastic dynamics governing the state, context, and reward processes. Figure~\ref{fig:contextual_mdp} illustrates a representative case in which the context is \emph{persistent}, i.e., $Z_{t+1} = Z_t$. The DAG shows how past states and actions influence future evolution while highlighting that the latent context $Z_t$ is a causal parent of both the next state $X_{t+1}$ and the reward $R_t$. Consequently, when $Z_t$ is unobserved, the behavioral data contain \emph{confounded} samples of the transition and reward pairs, as their observed variation reflects both the effects of the action and the hidden context.

Under the behavioral policy $\boldsymbol{g}^b$, the action $U_t$ may
depend on both $X_t$ and $Z_t$.
This alters the system’s SCM to introduce at each $t$, causal arrows from the context $Z_t$ and state $X_t$ to the action $U_t$ into the DAG.  
By contrast, under a state-based parameterized policy
$\boldsymbol{g}^{\boldsymbol{\theta}}$, the action depends only on $X_t$
and the arrow $Z_t \rightarrow U_t$ is absent.  
Thus, the behavioral and parameterized policies induce \emph{different}
SCMs, therefore different distributions over trajectories.
These issues motivate the need for an OPE
procedure capable of expressing the reward expectations.
In the next subsection, we reinterpret the C-MDP as a partially observed Markov decision process (POMDP), which reveals precisely how the trajectory distribution depends on hidden variables and enables the
proximal OPE construction that follows.


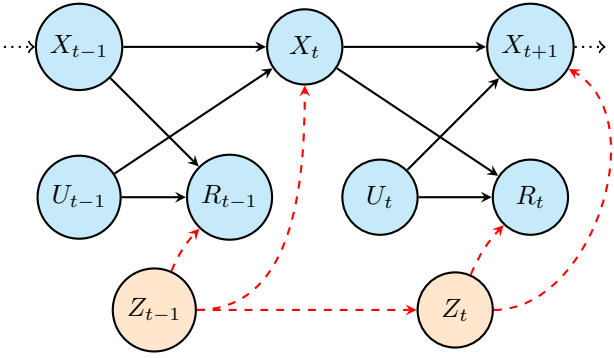
\begin{figure}
\centering
\begin{tikzpicture}[
    node distance=2.5cm and 3cm,
    every node/.style={circle,draw,minimum size=1cm,thick},
    state/.style={fill=cyan!20},
    context/.style={fill=orange!20},
    trans/.style={->,>=stealth,thick},
    influence/.style={->,>=stealth,dashed,red,thick},
    justdot/.style={->,dotted,thick}
]
\node[state] (X_t-1) at (0,0) {$X_{t-1}$};
\node[state] (X_t) at (3,0) {$X_t$};
\node[state] (X_t+1) at (6,0) {$X_{t+1}$};

\node[state] (U_t-1) at (0,-2) {$U_{t-1}$};
\node[state] (R_t-1) at (2,-2) {$R_{t-1}$};
\node[state] (U_t) at (4,-2) {$U_t$};
\node[state] (R_t) at (6,-2) {$R_t$};

\node[context] (Z_t-1) at (1,-3.5) {$Z_{t-1}$};
\node[context] (Z_t) at (5,-3.5) {$Z_t$};

\draw[trans] (X_t-1) -- (X_t);
\draw[trans] (X_t) -- (X_t+1);
\draw[trans] (X_t-1) -- (R_t-1);
\draw[trans] (X_t) -- (R_t);
\draw[trans] (U_t-1) -- (X_t);
\draw[trans] (U_t) -- (X_t+1);
\draw[trans] (U_t-1) -- (R_t-1);
\draw[trans] (U_t) -- (R_t);

\draw[influence] (Z_t-1) to [out=0,in=-90] (X_t);
\draw[influence] (Z_t-1) to [bend left=10] (R_t-1);
\draw[influence] (Z_t-1) to [out=0,in=180] (Z_t);
\draw[influence] (Z_t) to [out=0,in=-30] (X_t+1);
\draw[influence] (Z_t) to [bend left=10] (R_t);

\draw[justdot] ++(-1,0)-- (X_t-1);
\draw[justdot] (X_t+1) -- ++(1,0);
\end{tikzpicture}
\caption{Causal graph of the C-MDP.}
\label{fig:contextual_mdp}
\end{figure}


\subsection{POMDP Representation}

We recast the C-MDP as
a POMDP whose full state at time $t$ is defined by
$S_t = (X_t, Z_t),$
where $X_t$ is the observed component and $Z_t$ is the unobserved
context.
The set of full states is defined as $\mathcal{S}:= \mathcal{X} \times \mathcal{Z}$.
The agent observes only $Y_t = X_t, 
$ with the set of observations given by $\mathcal{Y}:=\mathcal{X}$. At each $t$, the agent receives a reward
\begin{align}
R_t = r(X_t, Z_t, U_t) = r(S_t, U_t).
\end{align}

A complete trajectory and the corresponding observable component up to time $t$ is therefore written as
\begin{align}
\tau_t = (s_{0:t},\, y_{0:t},\, u_{0:t}),
\qquad
\tau_t^o = (y_{0:t},\, u_{0:t})
\end{align}

Any parameterized policy
$\boldsymbol{g}^{\boldsymbol{\theta}}$ induces a probability measure
over the space of such trajectories.
This measure determines the
expected reward term in the Bellman consistency constraint
\eqref{eq:maxent_bellman_Q}.  
Thus, to enforce Bellman consistency for a state-based policy, we
express the reward distribution under the trajectory
distribution generated by $\boldsymbol{g}^{\boldsymbol{\theta}}$ in
terms of the observational data collected under $\boldsymbol{g}^{\boldsymbol{b}}$.

In the following subsections, we show the steps to achieve this via a proximal
OPE formulation.
Essentially, we rewrite the confounded reward expectation using
only observable quantities.  
Before presenting the proximal analysis, we introduce the probabilistic
notation used throughout.

\subsection{Notations}

We represent probabilities over latent states, observations, and rewards
using vector and matrix notation.  
For any realizations $s_t, s_{t+1} \in \mathcal{S}$ and any action
$u_t$, the transition probability $p(s_{t+1} \mid s_t, u_t)$ is
collected into the column vector
\[
P(S_{t+1}\mid s_t, u_t)
=
\bigl( p(s^1_{t+1}\mid s_t, u_t), \ldots,
       p(s^{|\mathcal{S}|}_{t+1}\mid s_t, u_t) \bigr)^{\!\top},
\]
and into the row vector
\[
P(s_{t+1}\mid S_t, u_t)
=
\bigl( p(s_{t+1}\mid s^1_t, u_t), \ldots,
       p(s_{t+1}\mid s^{|\mathcal{S}|}_t, u_t) \bigr).
\]
In our analysis, the multiplication of a pair of vectors of appropriate dimension is their \textit{scalar product}. 
The full transition kernel is expressed as the matrix
\[
P(S_{t+1}\mid S_t, U_t)
\in \mathbb{R}^{|\mathcal{S}|\times|\mathcal{S}|},
\]
with analogous conventions for matrices involving joint variables or
rectangular conditional distributions.
In our analysis, the multiplication of two matrices refers to their \textit{algebraic matrix multiplication}.
Our analysis also makes use of the conditional independencies among variables in the C-MDP, based on their causal relationship illustrated in Fig. \ref{fig:contextual_mdp}. 
As an example, the causal relationship corresponding to the Markovian evolution of the state given by
\[
p(s_{t+1} \mid s_{0:t}, u_{0:t})
=
p(s_{t+1}\mid s_t, u_t)
\]
is denoted compactly by
\[
s_{t+1} \perp\!\!\!\!\perp (s_{0:t-1}, u_{0:t-1})
\mid (s_t, u_t).
\]

For any policy $\boldsymbol{g}^{\boldsymbol{i}}, \boldsymbol{i}\in \{\boldsymbol{b},\boldsymbol{\theta}\}$, the distribution induced by $\boldsymbol{g}^{\boldsymbol{i}}$ on any tuple of random variables is denoted by a $p^{\boldsymbol{i}}(\cdot)$.

These conventions allow us to express the reward distribution as matrix products, which is crucial for the proximal OPE construction.

\subsection{OPE to Compute the Bellman Equation}

We now show how the Bellman equation in Problem~\ref{prob:maxcausalent} 
can be evaluated from the restricted dataset $\mathcal{D}^{\boldsymbol{b}}$ 
by using the POMDP formulation introduced earlier.  
Recall that the full state is $S_t = (X_t,Z_t)$, the observation is 
$Y_t = X_t$, and the reward is $R_t = r(S_t,U_t)$.  
Each policy $\boldsymbol{g}^{\boldsymbol{\theta}}$ induces a probability 
measure on full trajectories 
$\tau_t = (s_{0:t}, y_{0:t}, u_{0:t})$, 
while the dataset contains only 
$\tau_t^o = (y_{0:t}, u_{0:t})$.  
The expectations in the Bellman equation arise from this 
trajectory distribution, but the hidden context $Z_t$ makes the reward component confounded.
Recall the Bellman consistency constraint  
\begin{align}
    Q_t^{\boldsymbol{\theta},\boldsymbol{\phi}}(x_t,u_t)
    &= 
    \mathbb{E}^{\boldsymbol{\theta}}
    \!\left[
        r(X_t,Z_t,U_t)
        \,\big|\, X_t=x_t,U_t=u_t
    \right] \nonumber\\[2pt]
    &\quad +
    \mathbb{E}^{\boldsymbol{\phi}}
    \!\left[
        V_{t+1}^{\boldsymbol{\theta},\boldsymbol{\phi}}(X_{t+1})
        \,\big|\, X_t=x_t,U_t=u_t
    \right],
    \label{eq:Bellman_consistent_Q_function}
\end{align}
For a parameterized policy $\boldsymbol{g}^{\boldsymbol{\theta}}$, the second term depends only on the surrogate transition model 
$p^{\boldsymbol{\phi}}(x_{t+1}\mid x_t,u_t)$ and can be computed from 
observable transitions $(x_t,u_t,x_{t+1})$ in the dataset.

The first term,
$\mathbb{E}^{\boldsymbol{\theta}}[r(X_t,Z_t,U_t)\mid x_t,u_t]$,
depends on the unobserved context $Z_t$ and is therefore not identifiable 
by data-based averaging.  
Let $\mathcal{R}$ denote the finite set of possible reward realizations. 
Thus, we can express
\begin{align}\mathbb{E}^{\boldsymbol{\theta}} \left[ r(X_t,Z_t, U_t)\right]= \mathbb{E}^{\boldsymbol{\theta}} \left[ R_t\right]=\sum_{r_t\in \mathcal{R}} r_t \cdot p^{\boldsymbol{\theta}}(r_t).
\end{align}
We consider $p^{\boldsymbol{\theta}}(r_t)$ and use the law of total probability and Bayes' theorem to get
\begin{align}
p^{\boldsymbol{\theta}}(r_t)&=\sum_{\tau_{t}\in \mathcal{T}_{t}}\;p^{\boldsymbol{\theta}}(r_t|\tau_{t})\; p^{\boldsymbol{\theta}}(\tau_{t}),\\
&=\sum_{\tau_{t}\in \mathcal{T}_{t}}\;p^{\boldsymbol{\theta}}(r_t|s_{t},u_t)\;p^{\boldsymbol{\theta}}(\tau_{t}),\label{p_e_r_ierm2}
\end{align}
where, in the last equality, we use the conditional independence $r_{t} \perp\!\!\!\!\perp \tau_{t-1} | (s_{t},u_t)$ between the reward $r_t$ and past trajectory $\tau_{t-1}$.

Thus, in the POMDP representation, the reward component of the Bellman equation is confounded, while the future-value term remains identifiable 
from observable transitions.  
To proceed with model learning, we therefore require a method to express 
$\mathbb{E}^{\theta}[r(X_t,Z_t,U_t)\mid x_t,u_t]$ using only the observable 
trajectory prefixes $\tau_t^o$ available in $D^{\boldsymbol{b}}$.  
This is precisely where proximal learning becomes essential.  
Proximal OPE provides a method to replace the confounded reward expectation with observable quantities from data and a sequence of weight matrices that are proxy-based.
The past and current observations are used as proxies to recover the average effect of the hidden context on rewards without 
ever observing $Z_t$.

\subsection{Proximal Identification of the Confounded Reward Term}
\label{subsec:proximal_identification}

In this subsection, we extend the proximal learning framework of \cite{tennenholtz2020off} to the POMDP formulation developed above. Our objective is to express the confounded reward distribution in \eqref{p_e_r_ierm2}, originally defined with respect to the full latent state, in an equivalent form that depends only on observable variables. This reformulation is essential to ensure that the Bellman operator used in Problem~\ref{prob:maxcausalent} can be evaluated without access to the hidden context $Z_t$. To this end, we establish a sequence of lemmas that progressively eliminate the dependence on the full state $S_t$, thereby yielding a reward model identifiable from observational data. Throughout this subsection, we assume the following assumptions hold.

\begin{assumption}\label{assm_extraobs}
Access is available to an auxiliary ``null'' observation $Y_N$, distributed according to the prior over the initial observation $Y_0$. This variable serves as a proxy for the latent initial state $S_0$ and enables identification of the confounded reward distribution at $t = 0$.
\end{assumption}

\begin{assumption}\label{assm_tnvrtble}
For each $t=1,\dots,T$, the matrices  
$P^{\boldsymbol{b}}(Y_t \mid Y_{t-1}, u_t)$ and 
$P^{\boldsymbol{b}}(S_t \mid Y_{t-1}, u_t)$ 
are invertible.
\end{assumption}

This invertibility ensures that past and present observations carry 
sufficient information about the latent state $S_t$ so that hidden quantities can be inferred using proxy variables. 

\begin{remark}
We emphasize that Assumption~\ref{assm_tnvrtble} does \emph{not} hold in settings where the context $Z_t$ is resampled independently at each time step.  
In such cases, there is no temporal carryover of contextual information. Hence, the past observations cannot contain the proxy signal required to recover $Z_t$. 
\end{remark}

Each lemma isolates a structural step in the transformation from the original reward expansion to a fully observable representation.

\vspace{2mm}
\noindent\textbf{Step 1: Factorizing the trajectory distribution.}
We first characterize the trajectory distribution to distinguish policy-dependent factors from those governed exclusively by the system dynamics.
This separation will subsequently allow the application of proximal methods to the terms affected by unobserved variables.

\begin{lemma}[Expansion of the Policy--Dependent Trajectory Distribution]
\label{lem:traj_factorization}
At each $t$, for any trajectory $\tau_t = (u_{0:t}, y_{0:t}, s_{0:t})$, 
the policy-induced trajectory distribution satisfies
\begin{align}
\label{eq:lemma1_factorization}
\nonumber
p^{\boldsymbol{\theta}}(\tau_t)
&=
\Bigg( \prod_{k=0}^{t} p^{\boldsymbol{\theta}}(u_k \mid y_k) \Bigg)
\Bigg( \prod_{k=0}^{t} p^{\boldsymbol{\theta}}(y_k \mid s_k) \Bigg) \\
&\quad\cdot
\Bigg( \prod_{k=0}^{t-1} p^{\boldsymbol{\theta}}(s_{k+1} \mid s_k, u_k) \Bigg)
p^{\boldsymbol{\theta}}(s_0).
\end{align}
\end{lemma}

\begin{proof}
\begin{align}
    &\nonumber p^{\boldsymbol{\theta}}(\tau_{t})\\
    =&\;p^{\boldsymbol{\theta}}(u_{t}\mid y_{t},s_{t},\tau_{t-1})\;
       p^{\boldsymbol{\theta}}(y_{t},s_{t},\tau_{t-1}),\\
    =&\;p^{\boldsymbol{\theta}}(u_{t}\mid y_{t})\;
       p^{\boldsymbol{\theta}}(y_{t}\mid s_{t},\tau_{t-1})\;
       p^{\boldsymbol{\theta}}(s_{t},\tau_{t-1}),\label{p^e_tau_action}\\
    =&\;p^{\boldsymbol{\theta}}(u_{t}\mid y_{t})\;
       p^{\boldsymbol{\theta}}(y_{t}\mid s_{t})\;
       p^{\boldsymbol{\theta}}(s_{t}\mid\tau_{t-1})\;
       p^{\boldsymbol{\theta}}(\tau_{t-1}),\label{p^e_tau_obs}\\
    =&\;p^{\boldsymbol{\theta}}(u_{t}\mid y_{t})\;
       p^{\boldsymbol{\theta}}(y_{t}\mid s_{t})\;
       p^{\boldsymbol{\theta}}(s_{t}\mid s_{t-1},u_{t-1})\;
       p^{\boldsymbol{\theta}}(\tau_{t-1}),\label{p^e_tau_state}\\[2mm]
    \nonumber=&\;\Pi_{k=0}^{t}\,p^{\boldsymbol{\theta}}(u_k\mid y_k)\;
               \Pi_{k=0}^{t}\,p^{\boldsymbol{\theta}}(y_k\mid s_k)\;\\
    &\hspace{60pt}\cdot
      \Pi_{k=0}^{t-1}\,p^{\boldsymbol{\theta}}(s_{k+1}\mid s_k,u_k)\;
      p^{\boldsymbol{\theta}}(s_{0}),\label{p_e_r_ierm2_exp}
\end{align}
where, in \eqref{p^e_tau_action}, we use the property that the parameterized policy depends only on the current observation. 
In \eqref{p^e_tau_obs} we use the conditional independence 
$y_t \perp\!\!\!\!\perp \tau_{t-1}\mid s_t$ as the observation is dependent purely on the state.
Lastly, in \eqref{p^e_tau_state} we use the Markovian nature of the evolution of the full state.
\end{proof}

\vspace{1.5mm}
\noindent\textbf{Step 2: Decomposing the reward distribution.}
Next, we substitute Lemma~\ref{lem:traj_factorization} into the reward 
expansion and isolate all policy-dependent terms.

\begin{lemma}[Reward Distribution Decomposition]
\label{lem:reward_decomposition}
The reward distribution satisfies
\begin{align}
\label{eq:lemma2_reward_final}
\nonumber
p^{\boldsymbol{\theta}}(r_t)
&=
\sum_{\tau_t}
\Bigg( \prod_{k=0}^{t} p^{\boldsymbol{\theta}}(u_k \mid y_k) \Bigg)
\Bigg( \prod_{k=0}^{t} p^{\boldsymbol{b}}(y_k \mid s_k) \Bigg) \\
&\quad\cdot
p^{\boldsymbol{b}}(r_t \mid s_t, u_t)
\Bigg( \prod_{k=0}^{t-1}
p^{\boldsymbol{b}}(s_{k+1}, y_k \mid s_k, u_k) \Bigg)
p^{\boldsymbol{b}}(s_0),
\end{align}
and equivalently,
\begin{align}
\label{eq:lemma2_reward_vector}
\nonumber
p^{\boldsymbol{\theta}}(r_t)
&=
\sum_{\tau_t^o}
\Bigg( \prod_{k=0}^{t} p^{\boldsymbol{\theta}}(u_k \mid y_k) \Bigg)
p^{\boldsymbol{b}}(r_t, y_t \mid S_t, u_t) \\
&\quad\cdot
\Bigg( \prod_{k=0}^{t-1}
p^{\boldsymbol{b}}(S_{k+1}, y_k \mid S_k, u_k) \Bigg)
p^{\boldsymbol{b}}(S_0).
\end{align}
\end{lemma}

\begin{proof}
We begin by noting a key implication of Lemma \ref{lem:traj_factorization}.
Among the factors that appear in $p^{\boldsymbol{\theta}}(\tau_t)$, the 
only quantities that depend on the parameterized policy 
$\boldsymbol{g}^{\boldsymbol{\theta}}$ are the action probabilities 
$p^{\boldsymbol{\theta}}(u_k \mid y_k)$.  
In contrast, the observation model 
$p^{\boldsymbol{\theta}}(y_k \mid s_k)$ and the state transition model 
$p^{\boldsymbol{\theta}}(s_{k+1} \mid s_k, u_k)$
are structural properties of the underlying system and therefore do not change with the choice of policy.  
Hence, we can evaluate all the policy-independent terms based on any data-generating policy.
This invariance is crucial, as it allows us to substitute these terms with their behavioral counterparts 
$p^{\boldsymbol{b}}(\cdot)$ as we expand the reward distribution.

Starting from
\begin{align}
p^{\boldsymbol{\theta}}(r_t)
=
\sum_{\tau_t} 
p^{\boldsymbol{\theta}}(r_t\mid s_t,u_t)\,
p^{\boldsymbol{\theta}}(\tau_t), \end{align}
we substitute the factorization of Lemma \ref{lem:traj_factorization}.  
As a result of the invariance, at each $k=0, \ldots, t-1$, we rewrite
\begin{align}
p^{\boldsymbol{\theta}}(y_k\mid s_k)
&= p^{\boldsymbol{b}}(y_k\mid s_k),\\
p^{\boldsymbol{\theta}}(s_{k+1}\mid s_k,u_k)
&= p^{\boldsymbol{b}}(s_{k+1}\mid s_k,u_k).
\end{align}

Similarly, $p^{\boldsymbol{\theta}}(r_t\mid s_t,u_t)=p^{\boldsymbol{b}}(r_t\mid s_t,u_t)$ and $p^{\boldsymbol{\theta}}(s_0)=p^{\boldsymbol{b}}(s_0)$.
We substitute these behavioral policy-based terms and collect the factors to produce
\eqref{eq:lemma2_reward_final}.  
Finally, using our vector and matrix notation, we group each transition-observation pair into 
$p^{\boldsymbol{b}}(S_{k+1}, y_k\mid S_k,u_k)$ to produce 
\eqref{eq:lemma2_reward_vector}.
\end{proof}

At this point, the decomposition isolates the core challenge, which is the terms 
that still depend on the latent 
state $S_k$.
The next two lemmas show how these quantities can be rewritten using 
observable proxies.  
This removes all dependence on the unobserved context.

\vspace{2mm}
\noindent\textbf{Step 3: Introducing proxy variables.}
Next, we show that past and present observations $(Y_{t-1},Y_t)$ form valid 
proxies for $S_t$ and allow us to rewrite the terms based on the full-state.

\begin{lemma}[First Proxy: Incorporating $Y_{t-1}$]
\label{lem:first_proxy}
The matrix $P^{\boldsymbol{b}}(S_{t+1}, y_t \mid S_t, u_t)$ satisfies
\begin{align}
\label{eq:first_proxy}
\nonumber P^{\boldsymbol{b}}(S_{t+1}, &y_t \mid S_t, u_t)
\\
=&
P^{\boldsymbol{b}}(S_{t+1}, y_t \mid Y_{t-1}, u_t)\,
P^{\boldsymbol{b}}(S_t \mid Y_{t-1}, u_t)^{-1}.
\end{align}
\end{lemma}

\begin{proof}
Using the law of total probability,
\begin{align}
\nonumber
P^{\boldsymbol{b}}&(S_{t+1}, y_t \mid Y_{t-1}, u_t)\\
&=
P^{\boldsymbol{b}}(S_{t+1}, y_t \mid S_t, Y_{t-1}, u_t)\,
P^{\boldsymbol{b}}(S_t \mid Y_{t-1}, u_t),\\
&=
P^{\boldsymbol{b}}(S_{t+1}, y_t \mid S_t, u_t)\,
P^{\boldsymbol{b}}(S_t \mid Y_{t-1}, u_t),
\end{align}
where we use conditional independence 
$(S_{t+1},y_t)\perp\!\!\!\!\perp Y_{t-1} \mid (S_t,u_t)$.  
Under Assumption~\ref{assm_tnvrtble}, we right-multiply $P^{\boldsymbol{b}}(S_t \mid Y_{t-1}, u_t)^{-1}$ on both sides to obtain the desired expression.
\end{proof}

Thus, $Y_{t-1}$ can be incorporated as a proxy for $S_t$.  
To complete the construction, we also need to express $P^{\boldsymbol{b}}(S_t\mid 
Y_{t-1},u_t)$ using $Y_t$.

\begin{lemma}[Second Proxy: Incorporating the current observation $Y_t$]
\label{lem:second_proxy}
The following expression holds:
\begin{align}
P^{\boldsymbol{b}}(S_t \mid Y_{t-1}, u_t)
=
P^{\boldsymbol{b}}(Y_t \mid S_t, u_t)^{-1}\,
P^{\boldsymbol{b}}(Y_t \mid Y_{t-1}, u_t).
\end{align}
\end{lemma}

\begin{proof}
From the law of total probability,
\begin{align}
P^{\boldsymbol{b}}(Y_t \mid Y_{t-1}, u_t)
=
P^{\boldsymbol{b}}(Y_t \mid S_t, u_t)\,
P^{\boldsymbol{b}}(S_t \mid Y_{t-1}, u_t).
\end{align}
Under Assumption \ref{assm_tnvrtble}, we invert $P^{\boldsymbol{b}}(Y_t \mid S_t,u_t)$ 
to obtain the desired expression.
\end{proof}


\vspace{2mm}
\noindent\textbf{Step 5: Completing the proximal identification.}
Next, we combine all previous steps to express the confounded reward 
distribution using only observable matrices and the evaluation policy.

\begin{theorem}[Proximal Identification of the Confounded Reward Term]
\label{thm:proximal_reward_identification}
Under Assumptions~\ref{assm_extraobs} and~\ref{assm_tnvrtble},
\begin{align}
\label{eq:main_theorem_proximal}
\nonumber
p^{\boldsymbol{\theta}}(r_t)\hspace{25pt}\\
\hspace{-1pt}\nonumber=
\sum_{\tau_t^o}
\Bigg( \prod_{k=0}^{t} &p^{\boldsymbol{\theta}}(u_k \mid y_k) \Bigg)\\
&\cdot p^{\boldsymbol{b}}(r_t, y_{t} \mid Y_{t-1}, u_{t})\cdot
\Bigg( \prod_{k=0}^{t}
W_{k}(\tau^o_{k})\Bigg).
\end{align}
where, for each $k=1,\ldots,t$, the weight matrix is given by
\begin{align}
\nonumber&W_k(\tau^o_k)\\
&=
P^{\boldsymbol{b}}(Y_k \mid Y_{k-1}, u_k)^{-1}
\cdot
P^{\boldsymbol{b}}(Y_k, y_{k-1} \mid Y_{k-2}, u_{k-1}),
\end{align}
and for $k=0$ we set
\begin{align}
W_0(\tau^o_0)
=
P^{\boldsymbol{b}}(Y_0 \mid u_0, Y_{N})^{-1}
\cdot
P^{\boldsymbol{b}}(Y_0),
\end{align}
with $Y_{N}$ denoting the null observation associated with the prior 
in Assumption~\ref{assm_extraobs}.
Hence, the confounded reward expectation 
$\mathbb{E}^{\boldsymbol{\theta}}[r_t]$ is identifiable from observable 
data.
\end{theorem}

\begin{proof}
Lemma~\ref{lem:reward_decomposition} expresses the reward distribution as
\begin{align}
\nonumber
p^{\boldsymbol{\theta}}(r_t)
=
\sum_{\tau_t^o}
\Bigg( \prod_{k=0}^{t} &p^{\boldsymbol{\theta}}(u_k \mid y_k) \Bigg)
p^{\boldsymbol{b}}(r_t, y_t \mid S_t, u_t)\hspace{50pt} \\
&\quad\cdot
\Bigg( \prod_{k=0}^{t-1}
p^{\boldsymbol{b}}(S_{k+1}, y_k \mid S_k, u_k) \Bigg)
p^{\boldsymbol{b}}(S_0).
\end{align}
Thus, the only remaining dependence on the latent state appears through 
the product

\begin{align}
p^{\boldsymbol{b}}(r_t, y_t \mid S_t, u_t)\cdot
\Bigg( \prod_{k=0}^{t-1}
p^{\boldsymbol{b}}(S_{k+1}, y_k \mid S_k, u_k) \Bigg)
p^{\boldsymbol{b}}(S_0). \label{latent_state_term}
\end{align}

Our goal is to show that this product can be rewritten in terms of 
observable quantities via the weight matrices $W_k(\tau^o_k)$.

\vspace{1mm}
\noindent At each $k$, we apply the proxy-based decompositions from
Lemmas~\ref{lem:first_proxy} and~\ref{lem:second_proxy} to obtain
\begin{align}
\label{eq:after_proxy_yk}
\nonumber P^{\boldsymbol{b}}(&S_{k+1}, y_k \mid S_k, u_k)\hspace{35pt}\\[2mm]
\nonumber&=\;
P^{\boldsymbol{b}}(S_{k+1}, y_k \mid Y_{k-1}, u_k)\,\cdot
P^{\boldsymbol{b}}(Y_k \mid Y_{k-1}, u_k)^{-1}\,\\[2mm]
&\hspace{25pt}\cdot P^{\boldsymbol{b}}(Y_k \mid S_k, u_k).
\end{align}
Similarly, the reward-based term can also be decomposed to obtain 
\begin{align}
\label{eq:rewd_after_proxy_yk}
\nonumber P^{\boldsymbol{b}}&(r_{t}, y_{t} \mid S_t, u_t)\\[2mm]
\nonumber=&\;
P^{\boldsymbol{b}}(r_{t}, y_{t} \mid Y_{t-1}, u_t)\,\cdot
P^{\boldsymbol{b}}(Y_t \mid Y_{t-1}, u_t)^{-1}\,\\[2mm]
&\cdot P^{\boldsymbol{b}}(Y_t \mid S_t, u_t).
\end{align}


\vspace{1mm}
\noindent Next, we consider the product of two consecutive full-state dependent terms:
\begin{align}
\nonumber
P^{\boldsymbol{b}}&(S_{k+1}, y_k \mid S_k, u_k)\,\cdot
  P^{\boldsymbol{b}}(S_k, y_{k-1} \mid S_{k-1}, u_{k-1})
\\[2mm]
\nonumber
=&\;
P^{\boldsymbol{b}}(S_{k+1}, y_k \mid Y_{k-1}, u_k)\,\cdot
P^{\boldsymbol{b}}(Y_k \mid Y_{k-1}, u_k)^{-1}\,\\[2mm]
\nonumber&\cdot \boxed{P^{\boldsymbol{b}}(Y_k \mid S_k, u_k)
\cdot
P^{\boldsymbol{b}}(S_k, y_{k-1} \mid Y_{k-2}, u_{k-1})}\,\\[2mm]
&\cdot P^{\boldsymbol{b}}(Y_{k-1} \mid Y_{k-2}, u_{k-1})^{-1}\,\cdot
P^{\boldsymbol{b}}(Y_{k-1} \mid S_{k-1}, u_{k-1}), \label{conseec_prod_terms}
\end{align}
where we have applied the same proxy expansions to the second factor.
We now focus on the product of the highlighted inner full-state dependent terms in \eqref{conseec_prod_terms}. 
\begin{align*}
P^{\boldsymbol{b}}(Y_k \mid S_k, u_k)\,\cdot
P^{\boldsymbol{b}}(S_k, y_{k-1} \mid Y_{k-2}, u_{k-1}).         
\end{align*}

\noindent Consider the observable matrix 
$P^{\boldsymbol{b}}(Y_k, y_{k-1} \mid Y_{k-2}, u_{k-1})$.   
Using the law of total probability, we expand the observable matrix
\begin{align}
\nonumber P^{\boldsymbol{b}}(Y_k, y_{k-1} \mid Y_{k-2}, u_{k-1})\\[2mm]
\nonumber =\sum_{s_k \in \mathcal{S}} P^{\boldsymbol{b}}(Y_k \mid s_k, y_{k-1}, & Y_{k-2}, u_{k-1})\,\\
\cdot &P^{\boldsymbol{b}}(s_k, y_{k-1} \mid Y_{k-2}, u_{k-1}),
\end{align}
where, by the observation-model-based conditional independence
\begin{align*}
Y_k \perp\!\!\!\!\perp (y_{k-1}, Y_{k-2}) \mid (S_k, u_k),  
\end{align*}
the conditional probability simplifies to
\begin{align*}
P^{\boldsymbol{b}}(Y_k \mid s_k, y_{k-1}, Y_{k-2}, u_{k-1})
=
P^{\boldsymbol{b}}(Y_k \mid s_k, u_k).    
\end{align*}
Thus, the inner terms in \eqref{conseec_prod_terms} can now be rewritten to incorporate the proxies as
\begin{align}
\nonumber P^{\boldsymbol{b}}(Y_k \mid S_k, u_k)\, &\cdot
P^{\boldsymbol{b}}(S_k, y_{k-1} \mid Y_{k-2}, u_{k-1})\\[2mm]
&=P^{\boldsymbol{b}}(Y_k, y_{k-1} \mid Y_{k-2}, u_{k-1}).
\label{eq:weight_identity}
\end{align}

\medskip
\noindent\emph{Defining the weight matrix:}
Motivated by~\eqref{eq:after_proxy_yk} and~\eqref{eq:weight_identity}, we 
define, for $k\ge 1$, the observable weight matrix
\begin{align}
\nonumber&W_k(\tau^o_k)\\[2mm]
&=
P^{\boldsymbol{b}}(Y_k \mid Y_{k-1}, u_k)^{-1}
\cdot
P^{\boldsymbol{b}}(Y_k, y_{k-1} \mid Y_{k-2}, u_{k-1}).
\end{align}
Equation~\eqref{eq:weight_identity} shows that the factor involving the 
latent state $S_k$ in the inner product collapses to the observable term 
$P^{\boldsymbol{b}}(Y_k, y_{k-1} \mid Y_{k-2}, u_{k-1})$, while the 
remaining pre-multiplication by 
$P^{\boldsymbol{b}}(Y_k \mid Y_{k-1},u_k)^{-1}$ is also observable.  
We collect these terms into $W_k(\tau^o_k)$.
For $k=0$, the same construction is applied at the initial time step, using 
Assumption~\ref{assm_extraobs}, yields
\begin{align}
W_0(\tau^o_0)
=
P^{\boldsymbol{b}}(Y_0 \mid u_0, Y_{N})^{-1}\cdot
P^{\boldsymbol{b}}(Y_0),  \end{align}
where $Y_{N}$ denotes the null observation associated with the prior.

\medskip
Substituting~\eqref{eq:weight_identity} into the expression \ref{conseec_prod_terms} and using 
the definition of $W_k(\tau^o_k)$, we obtain
\begin{align*}
&P^{\boldsymbol{b}}(S_{k+1}, y_k \mid S_k, u_k)\, \cdot
P^{\boldsymbol{b}}(S_k, y_{k-1} \mid S_{k-1}, u_{k-1}) \\[2mm]
&=
P^{\boldsymbol{b}}(S_{k+1},y_k \mid Y_{k-1}, u_k)\, \cdot
W_k(\tau^o_k)\,\\[2mm]
&\quad
\cdot P^{\boldsymbol{b}}(Y_{k-1} \mid Y_{k-2}, u_{k-1})^{-1}\,\cdot
P^{\boldsymbol{b}}(Y_{k-1} \mid S_{k-1}, u_{k-1}),
\end{align*}
\noindent so that all dependence on $S_k$ has been removed and its effect is encoded 
by the observable matrix $W_k(\tau^o_k)$ and boundary terms involving 
$Y_{k-1}$ and $S_{k-1}$.

We apply this reduction iteratively for $k=0,\ldots,t$ to the full-state based product \eqref{latent_state_term} to obtain a chain of observable weight matrices
\begin{align}
p^{\boldsymbol{b}}(r_t, y_{t} \mid Y_{t-1}, u_{t})\cdot
\Bigg( \prod_{k=0}^{t}
W_{k}(\tau^o_{k})\Bigg). \label{latent_state_term_adjusted}
\end{align}

\noindent multiplying boundary terms that only depend on $(Y_0,S_0)$ and the 
initial prior in Assumption~\ref{assm_extraobs}.  
The remaining dependence on $S_0$ is captured by $p^{\boldsymbol{b}}(S_0)$ 
and is therefore compatible with the observable representation.

We Substitute the resulting expression for the product of the latent-state factors back into the decomposition of Lemma~\ref{lem:reward_decomposition} to produce \eqref{eq:main_theorem_proximal}, completing the proof.
\end{proof}

\section{Numerical Simulation}
\label{section:Numerical Simulation}

In this section, we illustrate the proposed framework on a synthetic clinical decision-making task. 
The example is motivated by intensive care unit (ICU) settings, where doctors prescribe treatments over time. 
In practice, treatments are repeatedly adjusted on the basis of recorded physiological measurements and unrecorded clinical judgment. 
We use this simulation to illustrate why it is necessary to learn a surrogate transition model that satisfies the proposed modified Bellman consistency equation to enable reliable model-based planning.

\subsection{Clinical Treatment Environment}

We consider a finite-horizon C-MDP that represents a short treatment episode for a single patient over a horizon of $T=9$ hours.
The discrete state space $\mathcal{X}$ aggregates physiological measurements like heart rate, blood pressure, and oxygen saturation into coarse health states to represent a severity score:
\begin{align*}
\mathcal{X} = \{ x^1, x^2, x^3, x^4 \},    
\end{align*}
where $x^1$ denotes a \emph{stable} state, $x^2$ a \emph{borderline} state, $x^3$ a \emph{deteriorating} state, and $x^4$ a \emph{critical}
state. 

At each time $t$, the doctor chooses an action  corresponding to the intensity of treatment, from the set 
\begin{align*}
\mathcal{U} = \{ u^1, u^2, u^3 \},    
\end{align*}
where $u^1$ is \emph{conservative treatment} (standard care), $u^2$ is \emph{moderate escalation}, and $u^3$ is \emph{aggressive escalation}.

The key feature is an unrecorded, patient-specific context $Z_t \in \mathcal{Z}$ that captures latent factors such as comorbidities and frailty, which the doctor can observe or access. We consider
\begin{align*}
 \mathcal{Z} = \{ z^L, z^H \},   
\end{align*}
where $z^L$ denotes a \emph{low-risk} profile and $z^H$ a \emph{high-risk} profile. We assume that the initial context $Z_0$ is drawn from a distribution $\nu$ and is persistent thereafter. Hence, $Z_{t+1} = Z_t$ for all $t$, which is consistent with the special-case causal graph illustrated in Fig.~\ref{fig:contextual_mdp}.
For example, an aggressive treatment $U_t=U^2$ may move a low-risk critical patient from $x^4$ to $x^2$ with higher probability than a high-risk patient in the same observed state. We design the transition kernel to reflect this intuition.
At each time $t$, the reward $R_t = r(X_t,Z_t,U_t)$ reflects clinical
status and treatment burden. We use
\begin{align*}
r(x,z,u) = r_{\text{health}}(x,z) - \lambda\, c(u),    
\end{align*}
where $r_{\text{health}}(x,z)$ assigns higher reward to less severe states and may depend on the latent context. The term $c(u)$ is indicative of the cost or burden of treatment intensity and weight $\lambda \in \mathbb{R}_{+}$ controls the relative importance of this cost. In our setting, $r_{\text{health}}$ is the dominant component of the reward, while the penalty $\lambda c(u)$ discourages the unnecessary use of high-intensity treatments.

The clinician follows a behavioral policy $\boldsymbol{g}^{\boldsymbol{b}} = \{ g^{\boldsymbol{b}}_{t} \}_{t=0}^{T-1}$ that depends on both $X_t$ and $Z_t$.
Thus, high-risk patients (with $z^H$) may receive aggressive treatment $u^2$ more often than low-risk patients, even when $X_t$ is the same.
\noindent The offline dataset $\mathcal{D}^b=\big\{    (x_0^i,u_0^i,r_0^i,\ldots,x_T^i,u_T^i,r_T^i)\big\}_{i=1}^N$
contains only observable state--action--reward tuples. The latent context $Z_t$ that influenced both $U_t$ and $R_t$ is never recorded.

    In the experiments, we specify concrete transition and reward parameters for each $(x,z,u)$ and choose $g_b$ to mimic a risk-aware
clinician. We then compare a naive model that uses data averaging and the proposed surrogate MDP with proximal reward correction, with an oracle model that observes $(X_t,Z_t)$.
Although our primary focus is on model quality, we note that in this setting the policy learned along with the proximal surrogate attains an expected return approximately $1.4\%$ higher than that of the naive learner (Monte Carlo estimate with $N = 1000$ episodes).

\begin{figure}
    \centering
    \includegraphics[width=\linewidth]{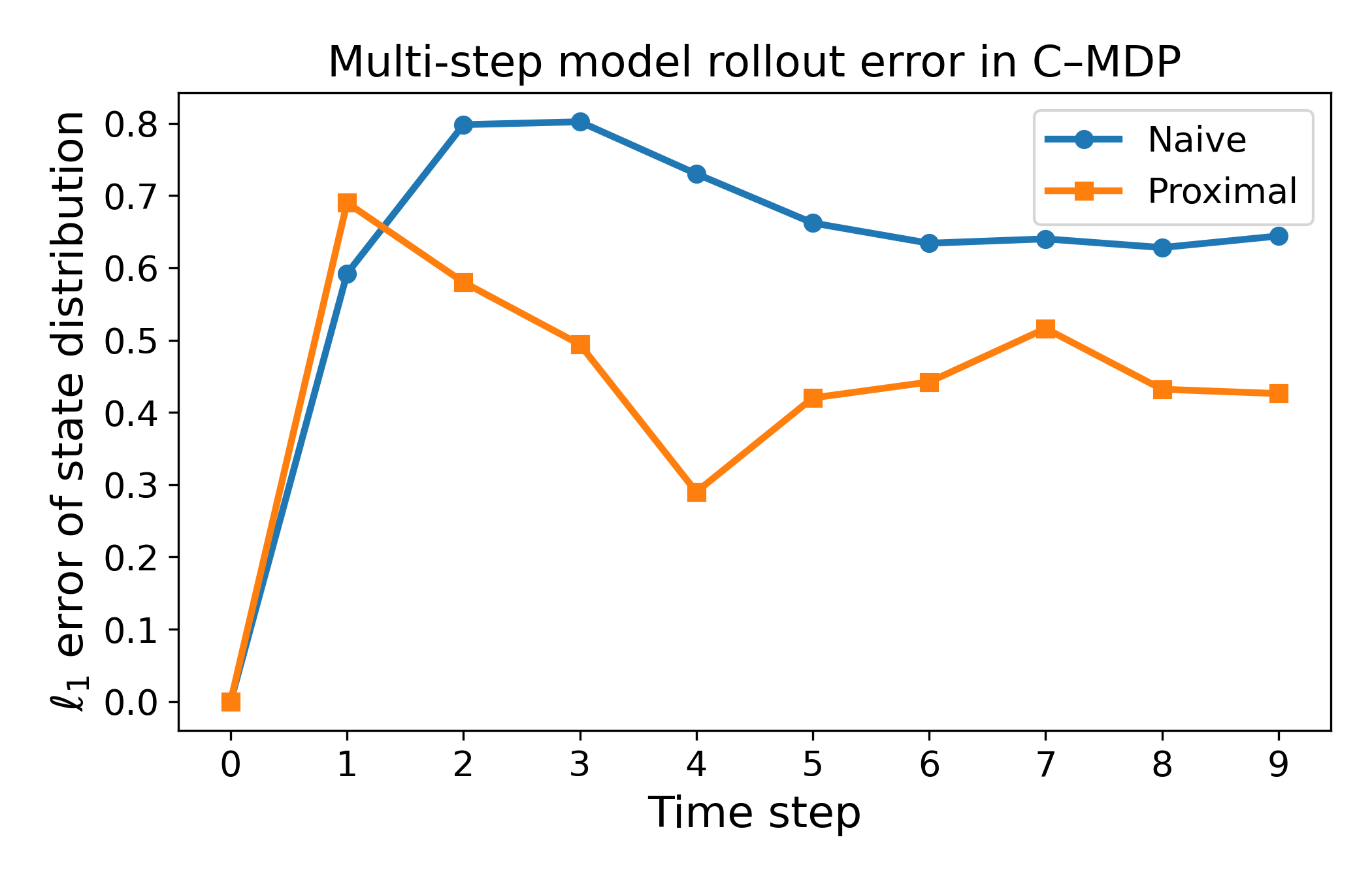}
    \caption{Comparison of multi-step rollout error}
    \label{fig:rollout_error}
\vspace{-15pt}
\end{figure}

Figure~\ref{fig:rollout_error} reports the multi-step rollout error in the C-MDP.
We measure the $\ell_1$ distance between the state distribution induced by each learned model and that of the true system under a fixed policy.
The empirical state distributions are estimated from $N = 1000$ Monte Carlo episodes.
The naive model attains a good one-step fit to the confounded data, but the rollout error rapidly accumulates and remains large over the horizon.
This behavior indicates that repeated application of the biased transition kernel leads to distorted evolution of the state.
In contrast, the surrogate model based on proximal learning exhibits consistently smaller $\ell_1$ error for $t \ge 2$. 
This suggests that the combination of the proximal reward construction and the surrogate dynamics yields a more accurate model of the system. 
In particular, the learned surrogate better reflects the long-horizon influence of the latent risk context on the observable state evolution, even though $Z_t$ is never observed in the offline dataset.

\section{Conclusion}
\label{sec:conclusion}

In this paper, we presented a model-based reinforcement learning framework for contextual MDPs in which the unobserved context confounds both the transition and reward mechanisms. By recasting the problem as a POMDP and analyzing it from a causal perspective, we showed that the reward component of the Bellman equation is not identifiable from observational data alone. To overcome this issue, we adapt a proximal off-policy evaluation method that reconstructs the confounded reward expectation using observable proxy variables under standard invertibility conditions, yielding an identifiable surrogate reward model compatible with state-based policies.

Combining this proximal correction with a behavior-averaged transition model and a maximum causal entropy formulation produced a Bellman-consistent surrogate MDP suitable for planning in the presence of hidden confounding. The results extend the applicability of model-based reinforcement learning to settings in which contextual information is unavailable or costly to obtain. A potential direction for future research should explore extensions to continuous state spaces and relaxations of the invertibility requirements.

\bibliographystyle{ieeetr}

\bibliography{References,Latest_IDS}

\end{document}